\documentclass[runningheads]{llncs}

\usepackage{eccv}

\usepackage{eccvabbrv}

\usepackage{graphicx}
\usepackage{booktabs}

\usepackage[accsupp]{axessibility}  

\usepackage{blindtext}
\usepackage{adjustbox}

\usepackage{pifont}
\usepackage{color, colortbl}

\usepackage{multirow}
\usepackage{threeparttable}

\usepackage{tabularx}

\usepackage{algorithm}
\usepackage{algorithmic}

\usepackage{amsmath}
\usepackage{thmtools,thm-restate}

\usepackage{hyperref}

\usepackage{orcidlink}
\usepackage{todonotes}
\usepackage{wrapfig}

\def\relu{{\text{ReLU }}}

\def\method{FTBC\xspace}
\newcommand{\cmark}{\color{ForestGreen}\ding{51}}
\newcommand{\xmark}{\color{Red}\ding{55}}

\newcommand{\gc}[1]{\cellcolor{gray!65} #1}

\begin{document}


\title{\method: Forward Temporal Bias Correction for Optimizing ANN-SNN Conversion}

\titlerunning{\method}




\author{
\author{Xiaofeng Wu\textsuperscript{*}\inst{1} \and
Velibor Bojkovic\textsuperscript{*}\inst{2} \and 
Bin Gu\inst{2} \and
Kun Suo\inst{3} \and
Kai Zou\inst{4}}
}


\institute{Faculty of Data Science, City University of Macau, Macau, China \email{xiaofengwu@cityu.edu.mo} \and
Department of Machine Learning, Mohamed bin Zayed University of Artificial Intelligence, Abu Dhabi, United Arab Emirates \email{bin.gu@mbzuai.ac.ae} \and
Department of Computer Science, Kennesaw State University, Georgia, USA \and
ProtagoLabs Inc} 





\maketitle

\renewcommand{\thefootnote}{\fnsymbol{footnote}} 
\footnotetext[1]{Equal contribution.}


\begin{abstract}

Spiking Neural Networks (SNNs) offer a promising avenue for energy-efficient computing compared with Artificial Neural Networks (ANNs), closely mirroring biological neural processes. However, this potential comes with inherent challenges in directly training SNNs through spatio-temporal backpropagation --- stemming from the temporal dynamics of spiking neurons and their discrete signal processing --- which necessitates alternative ways of training, most notably through ANN-SNN conversion. In this work, we introduce a lightweight Forward Temporal Bias Correction (\method) technique, aimed at enhancing conversion accuracy without the computational overhead. We ground our method on provided theoretical findings that through proper temporal bias calibration the expected error of ANN-SNN conversion can be reduced to be zero after each time step. We further propose a heuristic algorithm for finding the temporal bias only in the forward pass, thus eliminating the computational burden of backpropagation and we evaluate our method on CIFAR-10/100 and ImageNet datasets, achieving a notable increase in accuracy on all datasets. Codes are released at a GitHub repository.

\end{abstract}


\section{Introduction}
\label{sec:intro}

Spiking Neural Networks (SNNs), characterized by their ability to mimic the temporal dynamics of biological neurons through spiking and bursting activities, represent a significant leap forward in the quest for brain-like computational efficiency and accuracy~\cite{rathi2023exploring, mcculloch1943logical, hodgkin1952quantitative, lapicque1907recherches, DBLP:journals/nn/Maass97, DBLP:journals/tnn/Izhikevich03, fang2023spikingjelly}. Unlike traditional artificial neural networks (ANNs) that process information in a continuous manner, SNNs leverage discrete time events for data processing, which allows for a more energy-efficient computation by capturing the intrinsic temporal dynamics of neural information processing. This fundamental difference not only makes SNNs inherently more suited for tasks involving time-dependent data but also offers a promising path towards realizing computationally efficient, biologically plausible models that can operate on the edge~\cite{dalgaty2024mosaic, DBLP:journals/corr/abs-2401-01141}. 

Following the foundational exploration of SNNs, this evolution underscores the broad applicability and potential of SNNs into practical, real-world applications. Emerging applications like Spiking-YOLO~\cite{DBLP:conf/aaai/KimPNY20} for object detection, SpikingBERT~\cite{bal2023spikingbert} and SpikeGPT~\cite{DBLP:journals/corr/abs-2302-13939, wang2023masked} for natural language processing reveal SNNs' expanding role in tackling complex computational challenges. Bridging the gap between SNN models and neuromorphic hardware realizations, key innovations include TrueNorth's~\cite{merolla2014million,debole2019truenorth} scalable, non-von Neumann architecture, the Tianjic chip's~\cite{pei2019towards} hybrid approach blending SNNs and ANNs, and Loihi's~\cite{davies2018loihi} manycore processor with on-chip learning. These developments signify a paradigm shift towards massively parallel, energy-efficient computing platforms, emphasizing algorithm-hardware co-design.

\begin{figure}[ht]
    \centering
    \includegraphics[width=0.85 \linewidth]{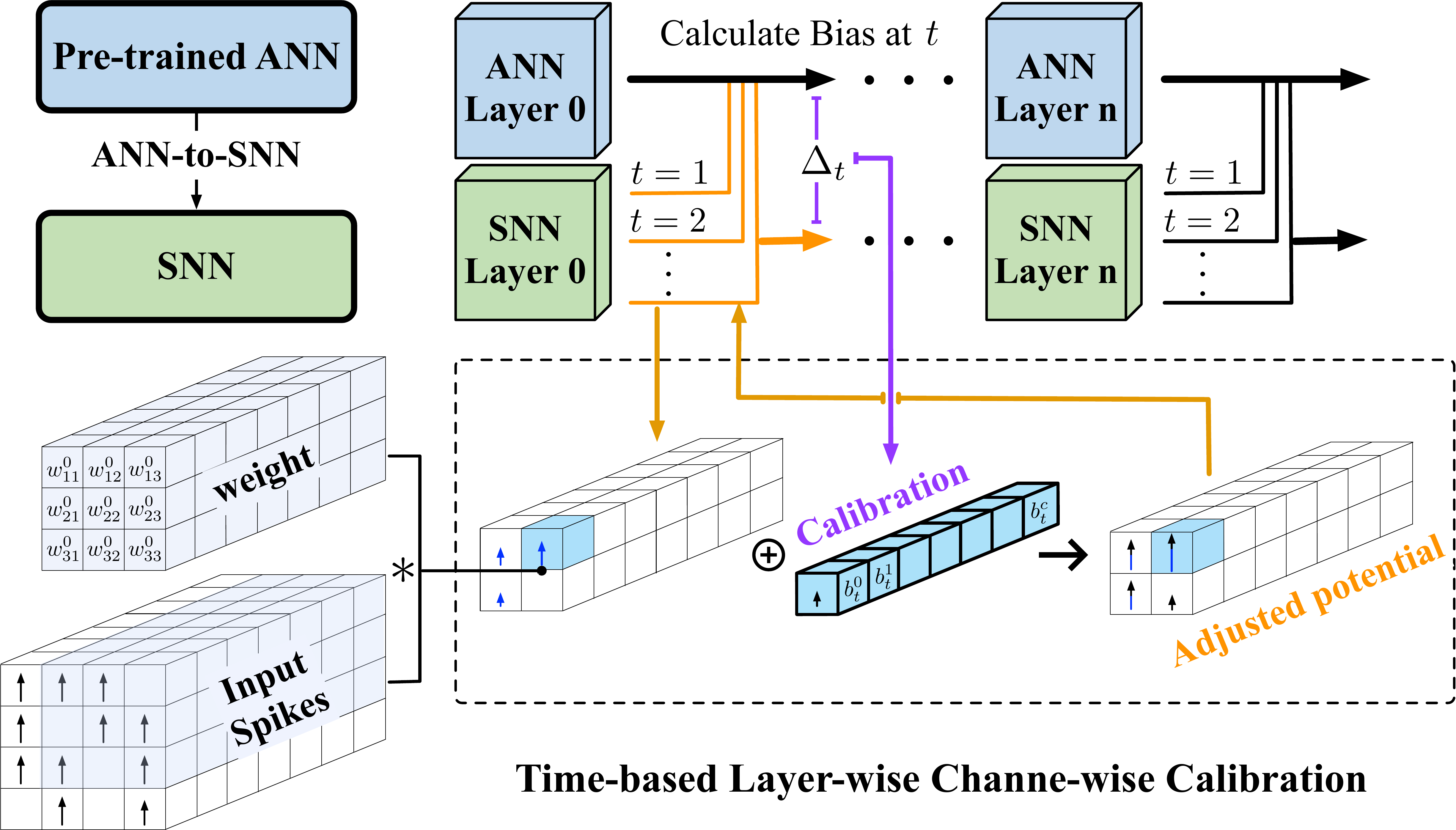}
    \caption{
        Overview of our proposed Forward Temporal Bias Correction (\method) method for ANN-SNN conversion. This approach calibrates time-based channel-wise bias terms (\(b_t\)) by dynamically adjusting membrance potential based on the temporal activation patterns observed in the pre-trained ANN. These adjustments are distributed across timesteps from \(t = 1\) to \(t = T\), ensuring the temporal precision of spike dynamics in SNNs is maintained.
    }
    \label{fig:overview}
\end{figure}

ANN-SNN conversion is a lightweight SNN training strategy among several others, including direct training, Spike Timing Dependent Plasticity (STDP)~\cite{bi2001synaptic, gerstner2002spiking, kempter1999hebbian, graupner2012calcium, rathi2018stdp}, hybrid learning~\cite{DBLP:conf/iclr/RathiSP020} and bio-plausible local learning~\cite{DBLP:journals/corr/abs-1811-10766, mostafa2018deep}. ANN-SNN conversion converts pre-trained ANNs to SNNs to align the firing rates of spiking neurons (Integrate-and-Fire (IF)~\cite{rueckauer2016theory}) with the Rectified Linear Unit (ReLU) activations in analog neurons. This conversion diverges into two subcategories: (1) methods that involve no modification to ANN structures prior to conversion but require bias and weight calibration afterwards such as SNN-Calibration~\cite{li2021free}; (2) methods that retrain ANNs from scratch due to the use of non-standard ReLU-like activations~\cite{DBLP:conf/iclr/DengG21, DBLP:conf/iclr/BuFDDY022, DBLP:conf/icml/JiangAMX023, DBLP:conf/aaai/HaoBD0Y23, DBLP:conf/iclr/HaoDB0Y23} to better approximate the IF function, such as quantization clip-floor-shift ReLU~\cite{DBLP:conf/iclr/BuFDDY022} and ReLU with threshold and shift~\cite{DBLP:conf/iclr/DengG21}. The conversion introduces three main errors~\cite{DBLP:conf/iclr/BuFDDY022}: clipping error, quantization error, and unevenness error. To further reduce conversion errors, other optimization strategies focus on weight normalization~\cite{Diehl2015Fast}, firing thresholds~\cite{rueckauer2017conversion, boj2024data}, optimizing IF biases~\cite{DBLP:conf/iclr/DengG21}, initial membrane potential~\cite{bu2022optimized, boj2024data}, and residual membrane potential~\cite{DBLP:conf/aaai/HaoBD0Y23} to enhance the accuracy and reduce the simulation steps.

Existing ANN-SNN conversion approaches have not effectively resolved unevenness error~\cite{DBLP:conf/iclr/BuFDDY022}. This error stems from the temporal variability in input spike sequences. First, SNN-Calibration~\cite{li2021free} uses the average of firing rates to match ANN's activations, which does not accurately capture the spike timing and exploit the temporal dynamics of spikes during the calibration process. When the timing of the arrival spikes changes, it can alter the output firing rates, leading to discrepancies in mapping ANN activations to SNN firing rates. Moreover, methods~\cite{DBLP:conf/iclr/BuFDDY022,DBLP:conf/iclr/DengG21} employing ReLU-like activations overlook the potential for post-conversion calibration, incorporating such calibration could notably enhance model accuracy. Also, conversion methods such as SlipReLU~\cite{DBLP:conf/icml/JiangAMX023}, QCFS~\cite{DBLP:conf/iclr/BuFDDY022} and RTS~\cite{DBLP:conf/iclr/DengG21} necessitate modifications to the original ANN's ReLU activation function to retrain the ANN models, potentially compromising ANN's accuracy~\cite{DBLP:conf/icml/JiangAMX023} and generalizability. Finally, SNN-Calibration~\cite{li2021free} with layer-wise weight calibration is more computationally demanding and memory-intensive than bias calibration. We seek to bridge this gap by offering a nuanced correction mechanism that dynamically adjusts biases in accordance with temporal activation discrepancies observed between ANNs and SNNs at each timestep, thus aiming to reduce the unevenness error and preserve the temporal fidelity. 

In this work, we propose the Forward Temporal Bias Correction, a post-conversion calibration approach to ANN-SNN conversion that addresses the differences between ANN outputs and SNN outputs at each time step. We aim to correct biases in an SNN by comparing the activations (outputs) of the SNN with those of an equivalent ANN at different timesteps. Figure~\ref{fig:overview} provides an overview of our proposed \method method for ANN-SNN conversion, illustrating how time-dependent channel-wise bias terms (\(b_t\)) are calibrated by dynamically adjusting them based on the temporal activation patterns observed in the pre-trained ANN. Adjustments are distributed across timesteps from \(t = 1\) to \(t = T\) by incorporating a temporal bias correction layer-by-layer at each timestep, ensuring the temporal precision of spike dynamics in SNNs is maintained. We demonstrate the effectiveness of \method through rigorous evaluations on CIFAR-10/100 and ImageNet datasets, showing superior performance compared to existing state-of-the-art. Our main contributions, encapsulated in this work, include:
\begin{itemize}
\item The introduction of the \method, an approach in ANN-SNN conversion that effectively addresses temporal bias, enhancing the accuracy of SNNs at each time step throughout the simulation time.
\item Theoretical foundation of the proposed method, culminating in the finding that through temporal bias calibration, the expected error of ANN-SNN conversion can be reduced to zero. We further provide a heuristic algorithm for finding temporal biases only in the forward pass.
\item The extensive experimental validation showing that our method surpasses existing conversion methods in terms of accuracy on CIFAR-10/100 and ImageNet datasets.
\end{itemize}

\section{Related Work}

The training methods of SNNs can be divided into two categories. ANN-SNN conversion involves reusing parameters of pre-trained ANNs to SNNs with the objective of minimizing accuracy degradation by aligning the ReLU activation outputs of ANNs with the firing rates of SNNs. The first work~\cite{cao2015spiking} demonstrates the feasibility of replacing ReLU activations with spiking neurons and paves the way for ANN-SNN conversion. Weight-normalization methods~\cite{Diehl2015Fast} extended this effort to bolster accuracy in converted SNNs across extensive time-steps. Reset-by-subtraction~\cite{rueckauer2017conversion} was introduced to effectively mitigate conversion errors and enhance the temporal accuracy of SNNs. A soft-reset mechanism~\cite{han2020rmp} was further refined to integrate dynamic threshold adjustments~\cite{stockl2021optimized, ho2021tcl, DBLP:journals/tnn/WuCZLLT23} to notably improve SNN performance. Rate-coding and time-coding~\cite{DBLP:conf/nips/KimKK20a} were adopted to train SNNs with fewer spikes. A tailored weight-normalization method~\cite{Sengupta2018Going} specifically for SNN operations, ensure near-lossless ANN-SNN conversion, while another direct conversion approach that addressed conversion errors by replacing ReLU activations in pre-trained ANNs with a Rate Norm Layer~\cite{DBLP:conf/ijcai/DingY0H21} and other novel activation functions~\cite{DBLP:journals/tnn/WuCZLLT23, DBLP:conf/iclr/BuFDDY022, DBLP:conf/icml/JiangAMX023}. The precision of conversion was further improved by calibrating weights and biases~\cite{li2021free, DBLP:conf/iclr/DengG21} through fine-tuning along with optimizing the initial membrane potential~\cite{li2021free} for enhanced performance in limited time-step scenarios. A quantization clip-floor-shift activation function (QCFS)~\cite{DBLP:conf/iclr/BuFDDY022} was introduced to minimize conversion error and unevenness, leading to improved SNN performance within shorter time-steps. The introduction of burst spikes~\cite{DBLP:conf/ijcai/Li022} and a memory function~\cite{wang2022signed}, alongside a unified optimization framework utilizing the SlipReLU activation function~\cite{DBLP:conf/icml/JiangAMX023} for zero conversion error within a specific shift value range, was noted. Additionally, strategies such as residual membrane potential~\cite{DBLP:conf/aaai/HaoBD0Y23} and initial membrane potential adjustment~\cite{DBLP:conf/iclr/HaoDB0Y23} have been developed to reduce unevenness error, significantly boosting SNN efficiency and accuracy. 

Direct training allows SNNs to account for precise spike timing and operate within a few timesteps. The success of direct training is attributed to the development of spatial-temporal backpropagation through time (BPTT) and surrogate gradient methods, enabling efficient training of SNNs' temporal spike sequence over time. However, these methods face challenges with deep architectures due to gradient instability, as well as high computational and memory costs that unfold over simulation timesteps. Various gradient-based methods leverage surrogate gradients~\cite{DBLP:conf/iclr/OConnorGRW18, zenke2018superspike, wu2018spatio, bellec2018long, fang2021deep, fang2021incorporating, zenke2021remarkable, mukhoty2024direct} to handle the non-differentiable spike functions. Direct training focuses on optimizing not only the synaptic weights but also dynamic parameters such as firing thresholds~\cite{wei2023temporal} and leaky factors~\cite{rathi2021diet}. Novel loss functions such as rate-based counting loss~\cite{zhu2024exploring} and distribution-based~\cite{guo2022recdis} loss are devised to provide adequate positive overall gradients and rectify the distribution of membrane potential during the propagation of binary spikes. Moreover, hybrid training methods~\cite{DBLP:journals/corr/abs-2205-07473} combine ANN-SNN conversion with BPTT to obtain higher performance under low latency. Recently, Ternary Spike~\cite{guo2023ternary} has been adopted to enhance information capacity without sacrificing energy efficiency. The reversible SNN~\cite{zhang2023memory} has been proposed to reduce the memory cost of intermediate activations and membrane potentials during training. Inspired by these approaches, this paper incorporates the concept of rectifying distribution of membrane potential with temporal biases and aims to reduce unevenness error and obtain high-accuracy SNNs.

\section{ANN-SNN conversion pipeline}

\subsection{IF spiking neuron}
Spiking Neural Networks (SNNs) often employ integrate-and-fire (IF) neurons, a popular choice in converting Artificial Neural Networks (ANNs) to SNNs. At each time step, these neurons integrate weighted spike inputs from preceding layers, gradually accumulating membrane potential. If this potential surpasses the membrane threshold, the neuron fires a single spike (weighted by the threshold) and resets its potential by subtracting the threshold value, as described in \cite{rueckauer2016theory}. This process repeats at discrete time steps throughout the entire simulation.  

To precisely represent the neuronal dynamics and the ANN-to-SNN conversion process, we introduce the following notation: $v[t]$ denotes the membrane potential (voltage) at time step $t$, $s[t]$ is a binary variable (1 or 0) indicating the presence of a spike at time step $t$, $W$ represents the network weight, while $\theta$ denotes the membrane threshold. Furthermore, we will use superscripts to signify the layer index, and subscripts to represent the neuron index within a specific layer. Finally, $T$ represents the total simulation time. With this, the temporal dynamics of an IF spiking neuron can be represented with
\begin{align}
    v^{(\ell)}[t+1] &= v^{(\ell)}[t]+W^{(\ell)}\theta^{(\ell-1)}s^{(\ell-1)}[t+1] - \theta^{(\ell)}s^{(\ell)}[t],\label{eq voltage}\\
    s^{(\ell)}[t+1] &= H(v^{(\ell)}[t+1]-\theta^{(\ell)}),
\end{align}
where $H$ is the Heaviside function.

\subsection{Basis of ANN-SNN conversion} \label{sec ANN-SNN pipeline}
\paragraph{General idea}
Summing up the previous equations through $t=1,\dots,T$, rearranging terms, and scaling with $T$, we obtain:
\begin{equation}\label{eq conversion}
    \frac{1}{T}\left(\theta^{(\ell)}\sum_{t=1}^Ts^{(\ell)}[t]\right) = W^{(\ell)}\frac{1}{T}\left(\theta^{(\ell-1)}\sum_{t=1}^Ts^{(\ell-1)}[t]\right) -\frac{1}{T}v^{(\ell)}[T],
\end{equation}
where $v^{(\ell)}[T]$ is the residue membrane potential after the spiking at the last time step $T$. At this point, we may notice that the  output of the spiking layer is non-negative, which suggests the similarity with the \relu activation function. Furthermore, if we suppose that the ANN output $a^{(\ell-1)}$ at layer $(\ell-1)$ is well approximated with the SNN output $\frac{1}{T}\left(\sum_{t=1}^Ts^{(\ell-1)}[t]\right)$, then applying the same weights $W^{(\ell)}$ on both outputs and using the \relu activation function on the ANN side, we obtain that the $(\ell)$ layer output of SNN approximates well the $(\ell)$ layer output of ANN, with error controlled by the term $\frac{1}{T}v^{(\ell)}[T]$. More precise relations are given in the supplementary material. 

\paragraph{Weights and threshold initialization} ANN-SNN conversion is based on using the pre-trained ANN model to initialize weights of the SNN model following the same architecture, such that the outputs of the SNN approximate well the ANN outputs, following the general idea above (see \cite{Diehl2015Fast, cao2015spiking}). Besides copying the weights, from equation \eqref{eq conversion} it becomes apparent that proper threshold initialization of SNNs to reduce the approximation error. This is usually based on some statistics coming from the distribution of the activation values of the ANN, such as taking the maximum activation or 99 percentile statistics \cite{rueckauer2017conversion}, where a simple grid search algorithm is used \cite{li2021free}, or a more elaborate statistical analysis is conducted like in \cite{boj2024data}.

\paragraph{Weight scaling} Furthermore, for the implementation of SNNs on neuromorphic hardware, one needs to scale the thresholds so that they become 1, which results in avoiding any multiply–accumulate (MAC) operation in the network and only accumulate (AC) operations remain, which is crucial for energy efficiency. In the work of ~\cite{diehl2016conversion, rueckauer2017conversion, li2021free}, weight normalization is described to standardize the representation of spikes across different layers in the neural network. Specifically, this technique converts spike values from \( \{0, \theta^{(\ell-1)}\} \) to a uniform \( \{0,1\} \) format in each layer by adjusting the weights according to the formula \( \mathbf{W}^{(\ell)} = ( \theta^{(\ell-1)} / \theta^{(\ell)} ) \mathbf{W}^{(\ell)} \). We note that we can scale the weights in the identical way in the initial ANN without changing its accuracy. 
\paragraph{Input encoding} We will consider constant encoding throughout the paper. That is, the input $x$ is passed through the SNN network at each time step during the simulation time, and it is the same input/sample on which the corresponding ANN has been trained. 

\paragraph{ANN-SNN pipeline} Finally, we will refer to the previous steps, that is: starting with a pre-trained ANN with \relu or \relu like activations, weight and thresholds initialization of an SNN following the same architecture, and finally, weight scaling for both ANN and SNN together with constant input encoding in SNNs as the \textbf{ANN-SNN pipeline}. 


\section{\method: Forward Temporal Bias Calibration}
We propose a time-dependent bias to compensate for the approximation errors when performing ANN-SNN conversion. In a nutshell, the voltage accumulation of spiking neurons, i.e. equation \eqref{eq voltage}, is changed to
\begin{equation}
    v^{(\ell)}[t+1] = v^{(\ell)}[t]+W^{(\ell)}\theta^{(\ell-1)}s^{(\ell-1)}[t+1] - \theta^{(\ell)}s^{(\ell)}[t] + b^{(\ell)}[t],
\end{equation}
where the bias $ b^{(\ell)}[t]$ will be calculated only in the forward manner, in such a way that the expected outputs of spiking neurons and the corresponding ANN neurons are equal. In what follows, we provide theoretical motivation for our approach, as well as the practical implementation of an algorithm to calculate terms $ b^{(\ell)}[t]$. We emphasize the computational efficiency of the whole procedure, as the temporal bias is calculated in the forward pass.

\subsection{Theoretical considerations}
The following general result is the basis for FTBC.
\begin{restatable}{proposition}{mainProp}\label{prop main}
    Let $\lambda$ be a continuous distribution with compact support $[a,c]$, and let $\alpha\in[0,1]$. Then, there exists $b^*\in \mathbb{R}$, such that
\begin{equation}\label{eq prop main}
    \mathbb{E}_{x\sim \lambda}[H(x+b^*-1)]=\alpha.
\end{equation}  
If $\lambda$ is strictly positive on $(a,c)$, then $b^*$ is unique.
\end{restatable}
We provide a proof and detailed discussion of the proposition in the supplementary material, but we note here that the previous result is a statement about SNNs in disguise. Namely, a typical context where we apply it is on the spiking neurons of an SNN obtained through ANN-SNN pipeline. Let us consider one such neuron in layer $(\ell)$ and position $i$. At a time step $t$, before the spiking, the values of the membrane potential $v^{(\ell)}_i[t]$, obtained when a training dataset is passed through the network, form a distribution, say $\lambda[t]$. Then, if we denote by $\bar{a}^{(\ell)}_i$ the expected output of the corresponding ANN neuron, Proposition \ref{prop main} tells us that there is a $b^{(\ell)*}_i[t]$, such that the expected output of the SNN neuron at time step $t$ is $\bar{a}^{(\ell)}_i$. This can be formalized in the following.

\begin{wrapfigure}{rh}{.45\textwidth}
    \begin{center}
    \includegraphics[width=0.45 \textwidth]{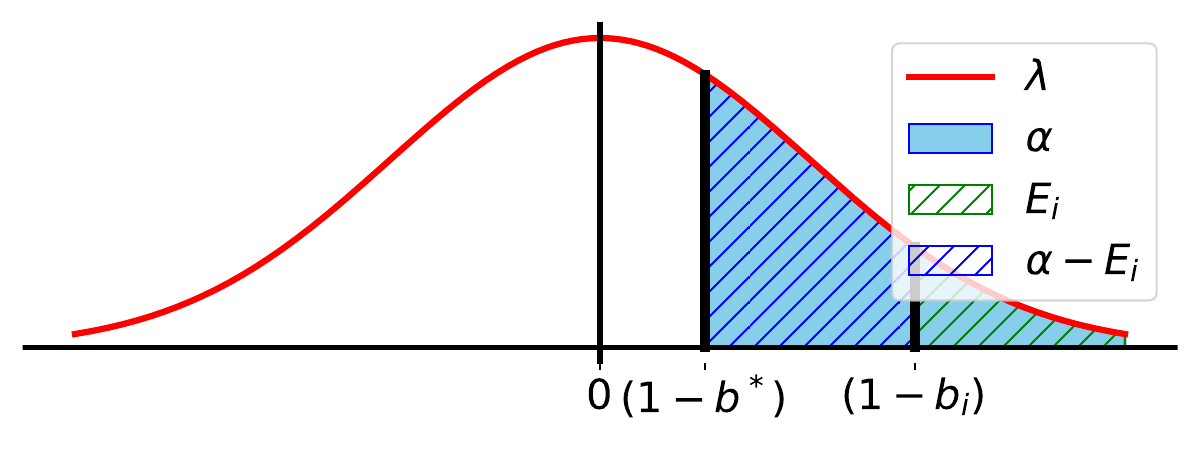}
    \end{center}
    \caption{
        A heuristic algorithm to iteratively find $b^*$: At each iteration, one estimates the difference between the desired $\alpha$ and current $E_i$ expected values, and estimates the next $b_{i+1}$ with $b_i+c(\alpha-E_i)$, where $c$ is some positive constant. 
    }
    \label{fig dist}
\end{wrapfigure}

\begin{restatable}{theorem}{mainThm}\label{thm main}
    Let $\mathcal{N}_A$ and $\mathcal{N}_S$ be an ANN and SNN, respectively, obtained through the ANN-SNN pipeline (see Section \ref{sec ANN-SNN pipeline}). Then, there exists temporal biases $b^{(\ell)}_{i}[t]$, $l=1,\dots, L$, $i=1,\dots, N_{\ell}$ such that the expected output of $\mathcal{N}_S$ after each time step coincides with the expected output of $\mathcal{N}_A$. In other words,
    $$
    \mathbb{E}_{x\sim \lambda}\left[\mathcal{N}_A(x)-\mathcal{N}_S(x)[t]\right] = 0.
    $$
\end{restatable}

The proof and other details concerning the previous result are provided in the supplementary material, but for the time being, we can unravel a special case that has been present in the literature. Namely, suppose we consider one ANN and one SNN neuron, where the ANN neuron receives input data with PDF $\lambda$ which is a uniform distribution with support $[0,1]$. Suppose that the SNN neuron receives the same inputs with constant encoding. Then, the temporal bias hinted in the Proposition \ref{prop main} and the previous theorem in the first time step is $b[1]=\frac{1}{2}$ (c.f. \cite{DBLP:conf/iclr/DengG21, bu2022optimized}).


\subsection{Practical aspects}

With motivation coming from the theoretical results from the previous section, we propose a heuristic algorithm for finding the temporal biases.  To start with, we explain the idea with the setting and notation of Proposition \ref{prop main}. Namely, suppose that at each iteration $i$, we receive a batch of samples from $\lambda$, say $B_i$, and suppose that the current estimation for $b^*$ is $b_i$.  We use the batch $B_i$ to estimate the left-hand side of \eqref{eq prop main}, $E_i:=\mathbb{E}_{x\sim B_i}[H(x+b_i-1)]$. Then, the next $b_{i+1}$ is obtained by adding to $b_i$ the signed difference $\alpha-E_i$ weighted with some positive hyperparemeter.  One may refer to Figure \ref{fig dist} for a pictorial representation of this idea. 

\begin{algorithm}[htbp]
\caption{Layer-wise Time-based Bias Correction}
\label{algo:bias-correction}
\begin{algorithmic}[1]
\REQUIRE ANN Model $\mathcal{M_\text{ANN}}$; SNN Model $\mathcal{M_\text{SNN}}$; Sample training dataset $\mathcal{D}$.
\FUNCTION{\textsc{BiasCorrEachTimeStep}($\mathcal{{\ell}_{\text{ANN}}}, \mathcal{{\ell}_{\text{SNN}}}, t, \mathbf{x}^{(\ell-1)}, \mathbf{s}^{(\ell-1)}(t) $)}    
    \STATE $\mathbf{a}^{(\ell)}, \mathbf{s}^{(\ell)}(t)  \gets \mathcal{{\ell}_{\text{ANN}}}(\mathbf{x}^{(\ell-1)}), \mathcal{{\ell}_{\text{SNN}}}(\mathbf{s}^{(\ell-1)}(t))$    
    \STATE {$\mathbf{B}^{(\ell)}(t) \gets \textsc{ChannelWiseMean} {\left( \mathbf{s}^{(\ell)}(t) - \mathbf{a}^{(\ell)} \right)} $}
    \STATE \textbf{return} $\mathbf{B}^{(\ell)}(t)$ 
\ENDFUNCTION
\FUNCTION{\textsc{BiasCorr}($\mathcal{M_{\text{ANN}}}, \mathcal{M_{\text{SNN}}}, \mathcal{D}, \alpha$)}
    \STATE $\mathbf{B}^{(\ell)}(t=1..T) \gets 0$ \COMMENT{Initialize time-based biases}
    \FORALL{$\mathcal{{\ell}_{\text{ANN}}}, \mathcal{{\ell}_{\text{SNN}}} \in \mathcal{M_{\text{ANN}}}, \mathcal{M_{\text{SNN}}}$}
        \FORALL{$(\mathbf{x}, \mathbf{y}) \in \mathcal{D}$}
            \FOR{$t = 1$ \textbf{to} $T$}
                \STATE Capture $\mathbf{x}^{(\ell-1)}$, $\mathbf{s}^{(\ell-1)}(t)$ from $\mathcal{M_\text{ANN}(\mathbf{x})}, \mathcal{M_\text{SNN}(\mathbf{x})}$
                \STATE $\mathbf{B'}^{(\ell)}(t) \gets \textsc{BiasCorrEachTimeStep}(\mathcal{{\ell}_{\text{ANN}}}, \mathcal{{\ell}_{\text{SNN}}}, t, \mathbf{x}^{(\ell-1)}, \mathbf{s}^{(\ell-1)}(t))$ 
                \STATE $\mathbf{B}^{(\ell)}(t) \gets \mathbf{B}^{(\ell)}(t) + \alpha \mathbf{B'}^{(\ell)}(t)$ 
                \STATE \COMMENT{Update the membrane potential of the SNN layer $\mathcal{\ell}_{\text{SNN}}$ immediately.}
            \ENDFOR
        \ENDFOR
    \ENDFOR    
\ENDFUNCTION
\STATE \COMMENT{Perform forward propagation for a single SNN layer}
\FUNCTION{\textsc{SpikeLayerForward}($\mathcal{\ell}, \mathbf{s}^{(\ell)}(t)$)} 
    \STATE $\mathbf{x}^{(\ell)} \gets \mathcal{\ell}(\mathbf{s}^{(\ell)}(t))$ \COMMENT{Compute activation from input}
    \STATE \COMMENT{Update membrane potential with activation}
    \STATE $\mathbf{v}_{\text{temp}}^{(\ell)}(t + 1) \gets \mathbf{v}^{(\ell)}(t) + \mathbf{x}^{(\ell)}$
    \STATE {\COMMENT{Apply channel-wise time-dependent bias $\mathbf{B}_t^{(\ell)}$}}
    \STATE {$\mathbf{v}_{\text{temp}}^{(\ell)}(t + 1) \gets \mathbf{v}_{\text{temp}}^{(\ell)}(t + 1) + \mathbf{B}_t^{(\ell)}$} 
    \STATE $\mathbf{s}^{(\ell+1)}(t) \gets \left( \mathbf{v}_{\text{temp}}^{(\ell)}(t + 1) \geq V_{\text{thr}}^{(\ell)} \right) \times V_{\text{thr}}^{(\ell)}$
    \STATE $\mathbf{v}^{(\ell)}(t + 1) \gets \mathbf{v}_{\text{temp}}^{(\ell)}(t) - \mathbf{s}^{(\ell+1)}(t)$ 
\STATE \textbf{return} $\mathbf{s}^{(\ell+1)}(t)$
\ENDFUNCTION
\end{algorithmic}
\end{algorithm}

Finally, the pseudocode provided in Algorithm~\ref{algo:bias-correction} outlines the procedure for our proposed \method. In an effort to better balance inference error and temporal error, we use time-based channel-wise bias calibration to obtain the optimal bias settings under different simulation lengths \( T \). A key component of our methodology involves capturing the output activations of the SNN network at every timestep, layer-wise and channel-wise, and further comparing this output against its corresponding counterpart ANN, focusing on aligning the two outputs more closely. By calculating the mean differences across channels for each time step, we obtain a detailed bias map that guides our calibration efforts.

\section{Experiments}

In this section, we verify the effectiveness and efficiency of our proposed \method. We compare our method with other state-of-the-art approaches for image classification tasks via converting ResNet-20, ResNet-34~\cite{He2016Deep}, VGG-16~\cite{simonyan2014very} on CIFAR-10~\cite{lecun1998gradient, krizhevsky2010cifar}, CIFAR-100~\cite{krizhevsky2009learning}, and ImageNet~\cite{deng2009imagenet}. We evaluate our approach compared with previous state-of-the-art ANN-SNN conversion methods, including TSC~\cite{han2020deep}, RMP~\cite{han2020rmp}, RNL~\cite{DBLP:conf/ijcai/DingY0H21}, ReLU-Threshold-Shift (RTS)~\cite{DBLP:conf/iclr/DengG21}, SNN Calibration with Advanced Pipeline (SNNC-AP)~\cite{li2021free}, ANN-SNN conversion with Quantization Clip-Floor-Shift activation function (QCFS)~\cite{DBLP:conf/iclr/BuFDDY022}, SNM~\cite{wang2022signed}, OPI~\cite{bu2022optimized}, and SlipReLU~\cite{DBLP:conf/icml/JiangAMX023}. 

In general, when initiating our SNNs, we follow the ANN-SNN pipeline outlined in Section \ref{sec ANN-SNN pipeline}. In combination with other methods, we adapt their particularities of initiating membrane potential and thresholds. We further calibrate such obtained models with various hyperparameter from Algorithm~\ref{algo:bias-correction} and batch-sizes. Then, we report the accuracies obtained after a certain number of calibration iterations (which do not have to be the best obtained). We further refer to the Appendix for details on all experimental setups and configurations. 

\subsection{Comparison with the State-of-the-art} \label{subsec:cmp_existing_methods}
\subsubsection{ImageNet dataset}
Table~\ref{tab:ann-snn-imagenet} shows the performance comparison of the proposed \method with the state-of-the-art ANN-SNN conversion methods on ImageNet.  When compared with the baselines, our method \method achieves superior performance throughout all simulation time steps. For example, on VGG-16 and $T=16$, our method achieves accuracy 71.19\% which some of the other baselines are only capable of reaching in 4 times longer latency $T=64$. Furthermore, in latency $T=128$, our method is capable of reaching the original ANN performance (with 0.7\% drop in VGG16 and 0.38\% drop in ResNet-34), which none of the baselines comes close to. 

\begin{table}[htbp]
\caption{Comparison between our proposed method \method and other ANN-SNN conversion methods on ImageNet dataset.}
\label{tab:ann-snn-imagenet}
\centering
\scalebox{0.95}
{
\begin{threeparttable}
\begin{tabular}{@{}ccccccccc@{}}
\toprule
Architecture & Method & ANN & T=4 & T=8 & T=16 & T=32 & T=64 & T=128\\ 
\toprule
\multirow{4}{*}{ResNet-34}
& SNNC-AP\tnote{*}  & 75.66 & -- & -- & -- & 64.54 & 71.12 & 73.45 \\
& QCFS & 74.32 & -- & -- & 59.35 & 69.37 & 72.35 & 73.15 \\
\cmidrule{2-9}
& \textbf{Ours\tnote{*} } & 75.66 & {13.70} & {38.55} & {60.68} & {70.88} & {74.29} & {75.28}  \\
\cmidrule{2-9}
& \textbf{Ours (+QCFS)} & 74.32 & {49.94} & {65.28} & {71.66} & {73.57} & {74.07} & {74.23} \\
\midrule
\multirow{7}{*}{VGG-16}
& SNNC-AP\tnote{*} & 75.36 & -- & -- & -- & 63.64 & 70.69 & 73.32 \\
& SNM\tnote{*} & 73.18 & --  & -- & -- & 64.78 & 71.50 & 72.86 \\
& RTS & 72.16 & -- & -- & 55.80 & 67.73 & 70.97 & 71.89 \\
& QCFS & 74.29 & -- & -- & 50.97 & 68.47 & 72.85 & 73.97 \\
& SlipReLU & 71.99 & -- & -- & 51.54 & 67.48 & 71.25 & 72.02 \\
& OPI\tnote{*} & 74.85 & -- & 6.25 & 36.02 & 64.70 & 72.47 & 74.24 \\
\cmidrule{2-9}
& \textbf{Ours\tnote{*}} & 75.36 &  {51.13} &  {64.2} &  {71.19} &  {73.89} & \gc {74.86} & \gc {75.29}  \\
\cmidrule{2-9}
& \textbf{Ours (+RTS)} & 72.16 & {29.61} & {55.22} & {67.14} & {70.74} & {71.86} & {72.13} \\
& \textbf{Ours (+QCFS)} & 73.91 & \gc {58.83} & \gc {69.31} & \gc {72.98} & \gc {74.05} & {74.16} & {74.21} \\
\bottomrule
\end{tabular}
\begin{tablenotes}
\item[*] Without modification to ReLU of ANNs.
\end{tablenotes}
\end{threeparttable}
}
\end{table}

Although the results show that \method as a standalone method is capable of achieving SOTA results, we further test its versatility and power in combination with other ANN-SNN techniques. For this, we chose RTS and QCFS methods, as they differ from other baselines in that they use modified \relu activations in order to reduce the conversion error. We refer again to Table \ref{tab:ann-snn-imagenet}.  For example, applying \method in conjunction with QCFS on ResNet-34 at $T=16$ results in a significant performance leap from $59.35\%$ to $71.66\%$, marking a $12.31\%$ improvement. A similar enhancement is observed for VGG-16 at the same timestep, where combining \method with QCFS boosts performance from $50.97\%$ to $72.98\%$, reflecting a $22.01\%$ increase. A similar pattern can be observed in the performance of \method together with the RTS method, where in VGG-16 and $T=16$, the original performance of RTS is improved by 11.14\%. This highlights the overall benefits of integrating \method with other optimization techniques, further solidifying its role as a comprehensive solution for ANN-SNN conversion challenges.

\begin{table}[htbp]
\caption{Comparison between our proposed method \method and other ANN-SNN conversion methods on CIFAR-100 dataset.}
\label{tab:ann-snn-cifar100}
\centering
\scalebox{0.85}
{
\begin{threeparttable}
\begin{tabular}{@{}ccccccccccc@{}}
\toprule
Architecture & Method & ANN & T=1 & T=2 & T=4 & T=8 & T=16 & T=32 & T=64 & T=128 \\ 
\toprule
\multirow{9}{*}{ResNet-20}
& TSC\tnote{*} & 68.72 & -- & -- & -- & -- & -- & -- & -- & 58.42 \\
& RMP\tnote{*} & 68.72 & -- & -- & -- & -- & -- & 27.64 & 46.91 & 57.69 \\
& SNNC-AP\tnote{*} & 77.16 & -- & -- & -- & -- & 76.32 & 77.29 & 77.73 & 77.63 \\
& RTS & 67.08 & -- & -- & -- & -- & 63.73 & 68.40 & 69.27 & 69.49 \\
& OPI\tnote{*} & 70.43 & -- & -- & -- & 23.09 & 52.34 & 67.18 & 69.96 & 70.51 \\
& QCFS\tnote{+} & 67.09 & 10.64 & 15.34 & 27.87 & 49.53 & 63.61 & 67.04 & 67.87 & 67.86 \\
& SlipReLU & 50.79 & \gc{48.12} & \gc{51.35} & 53.27 & 54.17 & 53.91 & 53.11 & 51.75 & --\\
\cmidrule{2-11}
& \textbf{Ours\tnote{*}} & 81.89 & 19.96 & 38.19 & \gc{58.08} & \gc{71.74} & \gc{78.80} & \gc{81.09} & \gc{81.79} & \gc{81.94} \\
\cmidrule{2-11}
& \textbf{Ours (+QCFS)} & 67.09 & 15.88 & 25.09 & 42.10 & 58.81 & 65.60 & 67.37 & 67.89 & 67.80 \\
\midrule
\multirow{9}{*}{VGG-16} 
& TSC\tnote{*} & 71.22 & -- & -- & -- & -- & -- & -- & -- & 69.86 \\
& SNM\tnote{*} & 74.13 & -- & -- & -- & -- & -- & 71.80 & 73.69 & 73.95 \\
& SNNC-AP\tnote{*} & 77.89 & -- & -- & -- & -- & -- & 73.55 & 77.10 & \gc{77.86} \\
& RTS & 70.62 & -- & -- & -- & -- & 65.94 & 69.80 & 70.35 & 70.55 \\
& OPI\tnote{*}  & 76.31 & -- & -- & -- & 60.49 & 70.72 & 74.82 & 75.97 & 76.25 \\
& QCFS\tnote{+} & 76.21 & 49.09 & 63.22 & 69.29 & 73.89 & 75.98 & \gc{76.53} & 76.54 & 76.60 \\
& SlipReLU & 68.46 & \gc{64.21} & {66.30} & 67.97 & 69.31 & 70.09 & 70.19 & 70.05 & -- \\
\cmidrule{2-11}
& \textbf{Ours\tnote{*}} & 77.87 & 32.79 & 48.99 & 60.68 & 69.52 & 74.05 & 76.39 & \gc{77.34} & 77.73 \\
\cmidrule{2-11}
& \textbf{Ours (+QCFS)} & 76.21 & 62.22 & \gc{67.77} & \gc{71.47} & \gc{75.12} & \gc{76.22} & 76.48 & 76.48 & 76.48 \\
\bottomrule
\end{tabular}
\begin{tablenotes}
\item[*] Without modification to ReLU of ANNs.
\item[+] Using authors' provided models and code. 
\end{tablenotes}
\end{threeparttable}
}
\end{table}

\subsubsection{CIFAR datasets}

We further evaluate the performance of \method method on CIFAR-100 and CIFAR-10 datasets. We show the results in Tables ~\ref{tab:ann-snn-cifar100} for CIFAR-100 and in the supplementary material for CIFAR-10. 

For CIFAR-100, we detect similar patterns as with the ImageNet dataset. Our method as standalone outperforms all the existing ANN-SNN conversion methods which use ANN models with \relu activation functions. The only exception is SlipReLU method which shows better performance at latencies $T=1,2$ on ResNet-20 and $T=1,2,4$ on VGG-16. However the SlipReLU method uses extensive hyperparameter search in order to obtain high performance, which induces significant computational costs but at the same time degrades ANN performance. When comparing our method with ANN-SNN conversion methods which use non-\relu  activations, e.g. QCFS and RTS, our method is constantly outperforming RTS on ResNet-20 and VGG16, while on QCFS the method alone is not able to outperform the baseline in lower latency on VGG16 (while \method is outperforming on ResNet-20). This is due to the continuous outputs of the \relu activation that we used in our ANN model, compared to the discrete outputs of steps functions used in QCFS. However, QCFS baseline suffers from necessity to train ANN models from scratch with custom activations, while \method is applicable to any ANN model with \relu-like activation. Furthermore, custom activation functions are sometimes sacrificing the ANN performance as can be seen from the corresponding ANN accuracies. 

We also show the performance of \method in combination with other methods, e.g. QCFS. In this case, we utilize the custom activation function of QCFS to yield high performance at low latency and further ameliorate the performance results. Once again we outperform all the existing methods (with the exception of SlipReLU at the aforementioned latencies) which highlights the effectiveness and applicability of \method.  

The results on the CIFAR-10 dataset presented in the supplementary material show similar patterns and performance with \method.


\subsubsection{Comparison with other types of SNN training methods}



We compare our approach with alternative SNN training methods, including Hybrid Training and Backpropagation Through Time (BPTT), benchmarking against methods such as Dual-Phase~\cite{DBLP:journals/corr/abs-2205-07473}, Diet-SNN~\cite{rathi2021diet}, RecDis-SNN~\cite{guo2022recdis}, HC-STDB~\cite{DBLP:conf/iclr/RathiSP020}, STBP-tdBN~\cite{zheng2021going}, TET~\cite{DBLP:conf/iclr/DengLZG22}, DSR~\cite{meng2022training}. The results presented in Table~\ref{tab:other-snn-training-methods} demonstrate that the performance of our method incorporating with QCFS achieves better accuracies of 94.67\% and 71.50\% on CIFAR-10 and CIFAR-100 for VGG-16 against BPTT and hybrid training with 4 timesteps, respectively. As in general direct training methods fall short in terms of accuracy compared to ANN-SNN conversion methods, we focus on reaching high accuracy at as low latency as possible, hence we couple \method with QCFS.

\begin{table}[htbp]
    \caption{Comparison with other types of SNN training methods}
    \label{tab:other-snn-training-methods}
    \renewcommand\arraystretch{1.}
	\centering
        \scalebox{0.80}{
        \resizebox{1.\textwidth}{!}{
    	\begin{tabular}{cccccc}
            \toprule
    	  Dataset & Architecture & Method & Type & Accuracy & Timestep \\ 
            \toprule
            \multirow{3}{*}{CIFAR-10} 
                & \multirow{3}{*}{VGG-16} 
                                        & Dual-Phase & Hybrid Training & 94.06 & 4 \\ 
                                        \cline{3-6} 
                &                       & Diet-SNN & Hybrid Training & 92.70 & 5 \\ 
                                        \cline{3-6} 
                &                       & \textbf{Our (+QCFS)} & ANN-SNN conversion & \gc{94.67} & 4 \\
            \hline
            \multirow{3}{*}{CIFAR-100} 
                & \multirow{4}{*}{VGG-16} 
                                        & Dual-Phase & Hybrid Training & 70.08 & 4 \\ 
                                        \cline{3-6} 
                &                       & Diet-SNN & Hybrid Training & 69.67 & 5 \\ 
                                        \cline{3-6} 
                &                       & RecDis-SNN & BPTT & 69.88 & 5 \\ 
                                        \cline{3-6} 
                &                       & \textbf{Our (+QCFS)} & ANN-SNN conversion & \gc{71.50} & 4 \\ 
            \hline
            \multirow{12}{*}{ImageNet} 
                & PreAct-ResNet-18      & DSR & Supervised learning & \gc{67.74} & 50 \\ 
                \cline{2-6}
                & \multirow{6}{*}{ResNet-34} 
                                        & HC-STDB & Hybrid Training & 61.48 & 250 \\ 
                                        \cline{3-6} 
                &                       & Dual-Phase & Hybrid Training & 61.20 & 8 \\
                                        \cline{3-6} 
                &                       & RecDis-SNN & BPTT & 67.33 & 6 \\ 
                                        \cline{3-6} 
                &                       & STBP-tdBN & BPTT & 63.72 & 6 \\ 
                                        \cline{3-6} 
                &                       & TET & BPTT & 64.79 & 6 \\ 
                                        \cline{3-6}
                &                       & \textbf{Our (+QCFS)} & ANN-SNN conversion & 62.26 & 6 \\ 
                &                       & \textbf{Our (+QCFS)} & ANN-SNN conversion & 65.28 & 8 \\ 
                \cline{2-6}
                & \multirow{3}{*}{VGG-16}
                                        & HC-STDB & Hybrid Training & 65.19 & 250 \\
                                        \cline{3-6} 
                &                       & Dual-Phase & Hybrid Training & 62.51 & 8 \\
                                        \cline{3-6}
                &                       & Diet-SNN & Hybrid Training & {69.00} & 5 \\
                                        \cline{3-6} 
                &                       & \textbf{Our (+QCFS)} & ANN-SNN conversion & \gc{69.31} & 8 \\ 
            \bottomrule
    	\end{tabular}
        }
        }
\end{table}

The application of our method to the ResNet-34 architecture on the ImageNet dataset at 6 timesteps results in an accuracy of 62.26\%, when compared with TET, STBP-tdBN, and RecDis-SNN utilizing the Backpropagation Through Time (BPTT) strategy. However, at 8 timesteps, our approach attains an accuracy of 65.28\%, surpassing the Dual-Phase method's 61.20\% and the HC-STDB method's 61.48\% at 250 timesteps with hybrid training. In the context of the VGG-16 architecture on the ImageNet dataset, our method demonstrates a higher performance relative to both the Dual-Phase and HC-STDB methods employing hybrid training at 8 timesteps, achieving an accuracy of 69.31\%. It should be noted that, unlike ANN-to-SNN conversion methods, back-propagation methods require the gradient to be propagated through the spatial and temporal dimensions during the training process. This propagation results in a substantial increase in memory and computational resource consumption.

\subsection{Comparison with the Self-Referential Baseline}

We first compare our method with the SNN Calibration (SNNC) approach, as both target post-calibration corrections. This comparison highlights the core similarities and differences in bias correction strategies. We further explore the impact of incorporating our post-calibration method for QCFS, RTS, 

\begin{table}[htbp]
\caption{Comparison of our algorithm with SNN Calibration~\cite{li2021free} on ImageNet. \textit{Use BN} means use Batch Norm layers to optimize ANN. Channel-wise Vth is used.}
\label{tab:imagenet-snnc}
\centering
\scalebox{0.8}{
\begin{tabular}{lcccccccc}
\toprule
Method & Use BN & ANN Acc. & $T=4$ & $T=8$ & $T=16$ & $T=32$ & $T=64$ & $T=128$ \\
\midrule
\multicolumn{9}{c}{\textbf{ResNet-34}} \\
\midrule
SNNC-LP & \xmark & 70.95 & - & - & - & 62.34 & 68.38 & 70.15 \\
Ours & \xmark & 70.95 & 13.97 & 37.55 & 57.44 & \gc{66.40} & \gc{69.36} & \gc{70.48} \\
\cmidrule{2-9}
SNNC-LP & \cmark & 75.66 & - & - & - & 50.21 & 63.66 & 68.89 \\
SNNC-AP & \cmark & 75.66 & - & - & - & 64.54 & 71.12 & 73.45 \\
Ours & \cmark & 75.66 & 13.70 & 38.55 & 60.68 & \gc{70.88} & \gc{74.29} & \gc{75.28} \\
\midrule
\multicolumn{9}{c}{\textbf{VGG-16}} \\
\midrule
SNNC-LP & \xmark & 72.40 & - & - & - & 69.30 & 71.12 & 71.85 \\
Ours & \xmark & 72.40 & 54.72 & 63.81 & 68.79 & \gc{71.11} & \gc{71.93} & \gc{72.25} \\
\cmidrule{2-9}
SNNC-LP & \cmark & 75.36 & - & - & - & 24.88 & 56.77 & 70.49 \\
SNNC-AP & \cmark & 75.36 & - & - & - & 63.64 & 70.69 & 73.32 \\
Ours & \cmark & 75.36 & 51.13 & 64.2 & 71.19 & \gc{73.89} & \gc{74.86} & \gc{75.29} \\
\bottomrule
\end{tabular}
}
\end{table}

SNNC-LP and our method share a fundamental objective of bias correction. However, SNNC-LP employs an average calibration over all time steps, utilizing \(\frac{1}{T}\sum_{t=1}^{T} s(t)^{l} \approx a^{l}\) to align the spiking neuron output with the analog neuron activation. In contrast, our method adopts a more granular approach, calibrating at each individual time step with \(s(t)^{l} \approx a^{l}\), which allows for a more precise and step-wise correction.

\begin{table}[htbp]
\caption{Comparison between the proposed method combining with QCFS on ImageNet dataset. L=16 for VGG-16 and L=8 for ResNet-34.}
\label{tab:imagenet-qcfs}
\centering
\renewcommand\arraystretch{1.}
\scalebox{0.8}
{
\begin{threeparttable}
\begin{tabular}{@{}cllllllllll@{}}\hline
Architecture & Method & ANN & T=1 & T=2 & T=4 & T=8 & T=16 & T=32 & T=64 & T=128\\ 
\hline
\multirow{2}{*}{ResNet-34}
&QCFS & 74.32 & 2.53 & 4.91 & 12.67 & 34.96 & 59.39 & 69.47 & 72.37 & 73.23 \\
&\textbf{Ours (+QCFS)} & 74.32 & \gc{16.95} & \gc{30.10} & \gc{49.94} & \gc{65.28} & \gc{71.66} & \gc{73.57} & \gc{74.07} & \gc{74.23} \\ \hline
\multirow{2}{*}{VGG-16}
&QCFS & 73.91 & 0.72 & 1.50 & 4.68 & 19.10 & 50.87 & 68.35 & 72.82 & 73.98 \\
&\textbf{Ours (+QCFS)} & 73.91 & \gc{25.75} & \gc{41.10} & \gc{58.83} & \gc{69.31} & \gc{72.98} & \gc{74.05} & \gc{74.16} & \gc{74.21} \\  \hline
\end{tabular}
\end{threeparttable}
}
\end{table}

Table~\ref{tab:imagenet-snnc} demonstrates substantial performance improvements, showcasing the effectiveness of our calibration approach. Specifically, when examining the results at \(T=32\) for VGG-16 architectures with batch normalization, our method significantly outperforms SNNC-LP, enhancing the accuracy from 24.88\% to 73.89\% and from 0.254\% to 39.49\% in different scenarios. This stark improvement underscores the advantage of our step-by-step calibration method over the averaged approach of SNNC-LP, particularly in scenarios requiring precise temporal resolution. Results for CIFAR-10 and CIFAR-100 are provided in the appendix.

Our study further assesses the effects of integrating our post-calibration method into QCFS and RTS. Table~\ref{tab:imagenet-qcfs} confirms the efficacy of our method in QCFS, illustrating steady accuracy gains for ResNet-34 and VGG-16 on ImageNet across all timesteps, particularly at lower T values. Notably, VGG-16's accuracy at T=2 soared from 1.50\% to 41.10\%, and ResNet-34's from 4.91\% to 30.10\%. Additionally, we conducted hyperparameter optimization experiments on CIFAR-10 and CIFAR-100, varying the batch size, with results detailed in the appendix. Similar tables are provided for the RTS method in the appendix.

\subsection{Stability of calibration}
We further tested the stability and convergence of \method in dependence on batch iterations used. The results are presented in Figure \ref{fig:dist}. We note that the accuracy in higher latency stabilizes faster, which is due to the accumulated bias and membrane potential through time. Furthermore, the accuracy at a particular higher latency may drop in the iterations when the lower latencies start to "pick-up" and improve the performance. However, these ripples in performance are smoothed quickly over the next few iterations as can be referred to from the figure. Final minor oscillations in the accuracies are due to the hyperparameter $\alpha$ (which may prevent the algorithm from finding the optimal biases), and stochasticity coming from the batch sampling.  

\begin{figure}[htbp]
    \centering
    \begin{subfigure}{0.48\textwidth}
        \centering
        \includegraphics[width=\textwidth]{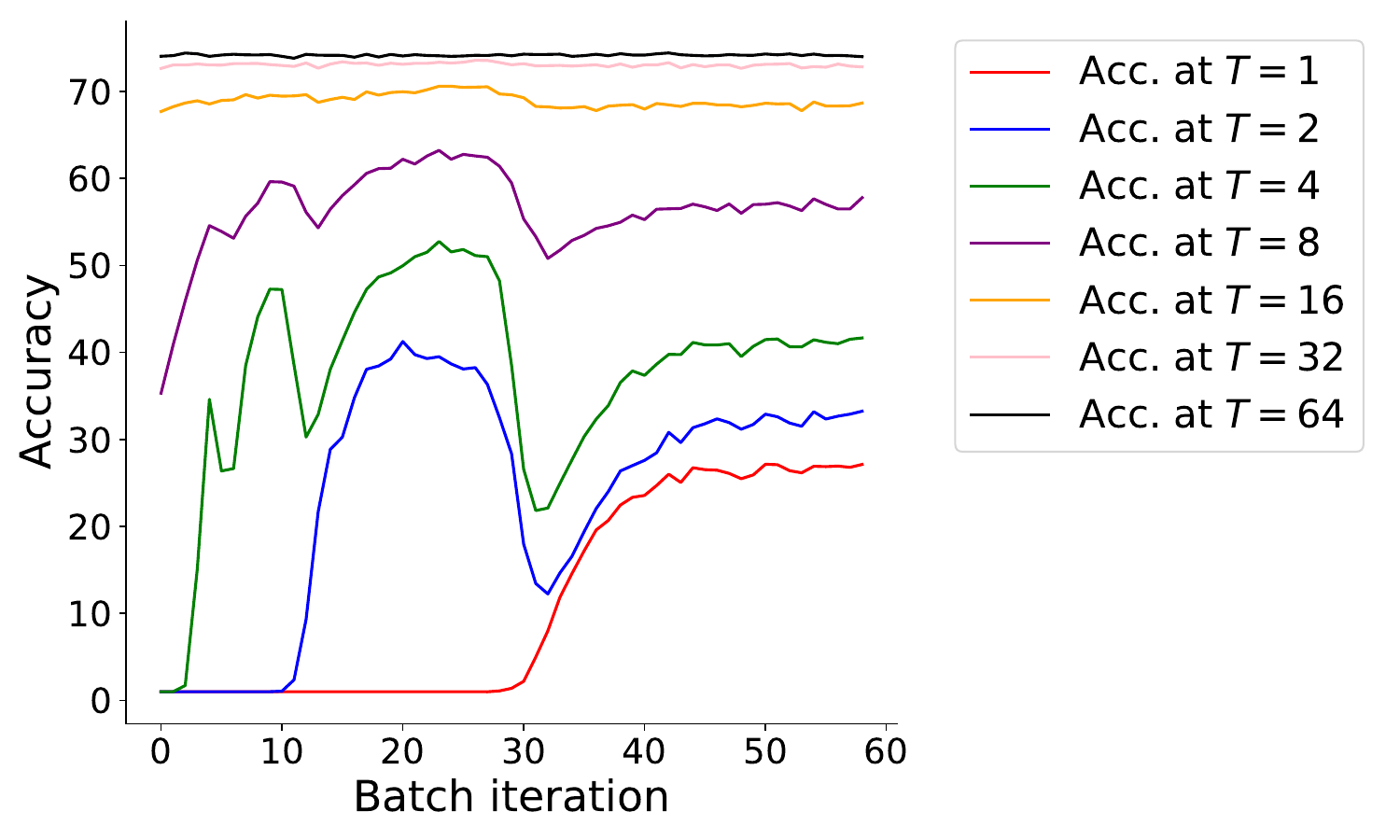}
        \caption{$\alpha=0.5$}
        \label{fig:stable2}
    \end{subfigure}
    \hfill
    \begin{subfigure}{0.48\textwidth}
        \centering
        \includegraphics[width=\textwidth]{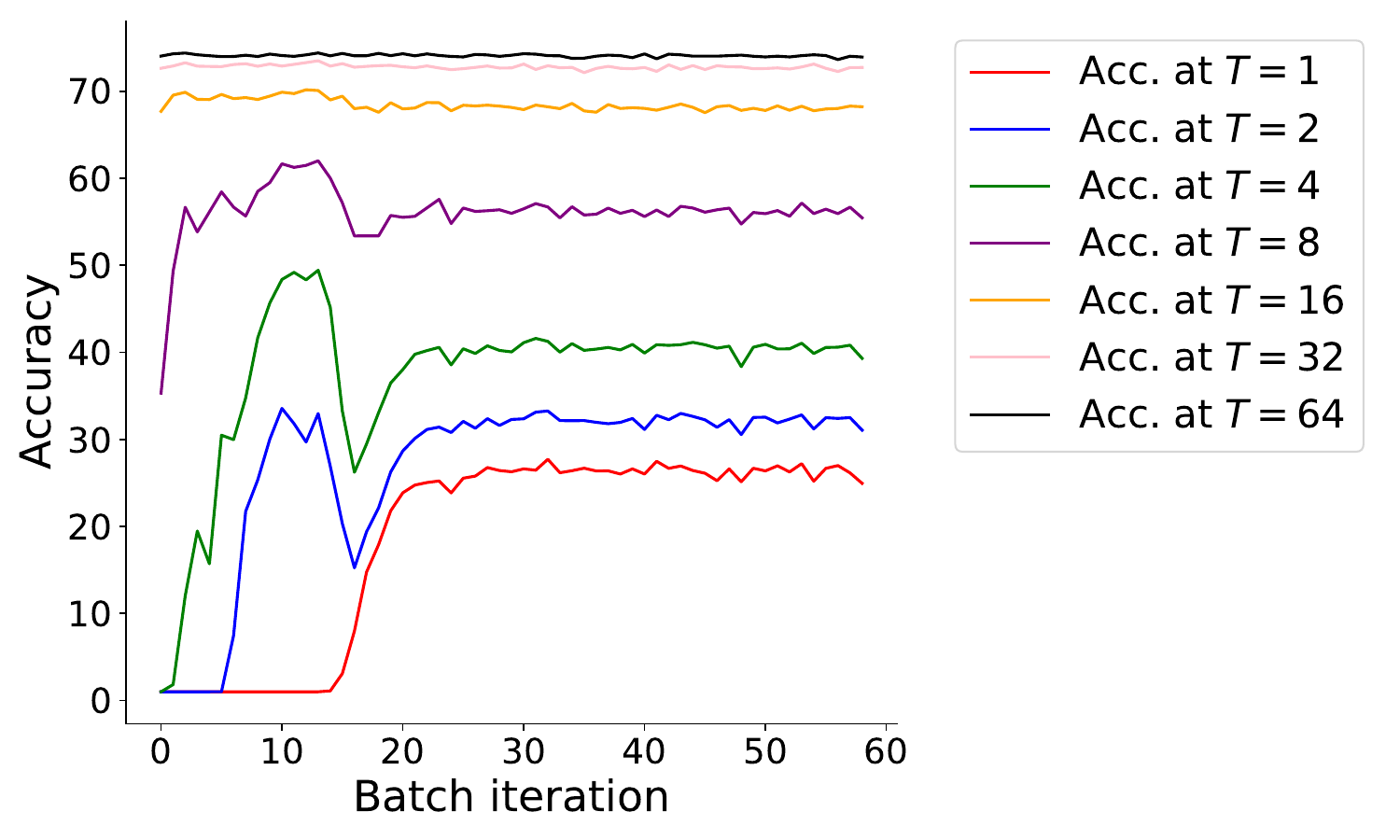}
        \caption{$\alpha=1.0$}
        \label{fig:stable1}
    \end{subfigure}
    \caption{Change of accuracy at various time steps with batch iteration. The reported accuracies are for a VGG16 model trained on CIFAR100 dataset and \(\alpha=0.5\) and \(\alpha=1.0\) as a hyperparameter in calibration (see Algorithm \ref{algo:bias-correction}).}
    \label{fig:dist}
\end{figure}

\section{Conclusion}

In this study, we introduced the Forward Temporal Bias Correction (\method) for ANN-SNN conversion, incorporating temporal dynamics into bias correction to address temporal bias effectively. Through theoretical analysis and extensive experimentation on CIFAR-10/100 and ImageNet datasets, we not only quantified conversion errors but also offered practical solutions to mitigate them, underscoring the potential of \method in bridging the gap between ANNs and SNNs. Our work emphasizes the importance of temporal dynamics in neural computation, presenting a robust framework for enhancing SNN performance while maintaining computational efficiency. 


%
%



\appendix

\section{Discussion and proofs of theoretical results}
\subsection{ANN-SNN pipeline}
\paragraph{ANN-SNN pipeline - Architectures} When considering ANN-SNN conversion, one starts with an ANN model with architecture that contains layers which are linear, or become linear, during the inference. In particular, these include fully connected, convolutional, batch normalization and average pooling layers. We note that the skip connections, in the architectures can also be represented as a linear transformation (identity) of an output of one layer to the input of the connected one. 

\paragraph{ANN-SNN pipeline - ANN activations} In this article, we considered ANN models with \relu or \relu-like activations functions. In particular, we considered thresholded \relu (associated with RTS method) defined as 
$$
\text{TReLU}(x)=\min(\text{ReLU}(x),a),
$$
where $a$ is some positive constant. Depending on the dataset for which the model was trained, we take $a$ to be 1 or 2. We note that one may also recognize the \textit{clamp} function in TReLU.\\
Another function we consider is the (shifted) stairs function (associated with QCFS method) which is defined as
$$
\text{Stairs}(x) = \frac{1}{l}\lfloor y\cdot l+0.5\rfloor,
$$
where $y=clamp(x,0,1)$, and $l$ is a positive integer (the number of steps). In the experiments, the parameter $l$ depends on the datasets on which the model was trained, and we followed the QCFS method experimental setting when dealing with this function.
\subsection{Proofs of theoretical statements}
We next start with proof of the Proposition \ref{prop main}.
\mainProp*
\begin{proof}
    Given some $b\in \mathbb{R}$, we have 
    $$
    \mathbb{E}_{x\sim \lambda}[H(x+b-1)]=\int_{x>-(b-1)} d\Lambda = 1-\Lambda(-(b-1)),
    $$
where $\Lambda$ is CDF corresponding to $\lambda$. As $\Lambda$ is a continuous function with $\Lambda(a)=0$ and $\Lambda(c)=1$, the results follows. Moreover, if $\lambda$ is strictly positive on $(a,c)$, $\Lambda$ is strictly increasing, and the uniqueness follows.
\end{proof}

\setcounter{theorem}{0}

Here is a more precise version of Theorem \ref{thm main}.
\begin{theorem}\label{thm precise}
    Let $\mathcal{N}_A$ and $\mathcal{N}_S$ be an ANN and SNN, respectively, obtained through the ANN-SNN pipeline (see Section \ref{sec ANN-SNN pipeline}). Suppose that the input is a continuous random variable, and that at some layer $l$ and some time step $t$, the membrane potential $v^{(l)}_i[t]$, $i=1,\dots,N_l$, before firing, seen as a function of the input, is a continuous random variable. Then, there exists temporal biases $b^{(l)}_i[t]$ such that the expected output of $\mathcal{N}_S$ at layer $l$ and time step $t$ is equal to the expected output of $\mathcal{N}_A$.
\end{theorem}   
\begin{proof}
    Let $a^{(l)}$ be the layer $l$ expected output of $\mathcal{N}_A$. Then, we can directly apply Proposition \ref{prop main} to the PDF of $v^{(l)}_i[t]$ and $a^{(l)}_i$, $i=1,\dots,N_l$ to conclude the existence of the asserted $b^{(l)}_i[t]$. 
\end{proof}

\subsection{Theoretical results in practice} 

\paragraph{The conditions.} The conditions of the previous theorem seem to be too strong to be applicable. Even more so having in mind that our input random variables are discrete (as all the data is). However, on a more abstract level, one can read the theoretical results as: 

\textit{As long as the distribution of the membrane potential is sufficiently diverse, or dense in some open segment, one can approximate the expected output of the ANN with the expected output of the corresponding SNN.} 

That this, weaker condition, is really the satisfied in practice, we plotted the histograms of the membrane potential of neurons in 4th and 7th spiking layer in a VGG16 model (trained on CIFAR100 dataset) in Figure \ref{fig mem dist}. In Figure \ref{fig mem dist} (a), the model has been calibrated for temporal bias for 30 iterations (to achieve stability as is shown in Figure \ref{fig:dist} for all the previous layers (1-3), and the histogram represents the distribution of the membrane potential of a neuron in a layer that has not been calibrated yet. Similarly, in Figure \ref{fig mem dist} (b), the model has been calibrated for temporal bias for 30 iterations, for all the layers 1-5, and in the layer 6 for all the time steps $t=1,2,3$. Then, we recorded the distribution of the membrane potential of a spiking neuron in the layer at time step $t=4$ before spiking. As hopefully both the histograms show, the membrane potential is distributed with high diversity and density.

\begin{figure}[htbp]
    \centering
    \begin{subfigure}{0.48\textwidth}
        \centering
        \includegraphics[width=\textwidth]{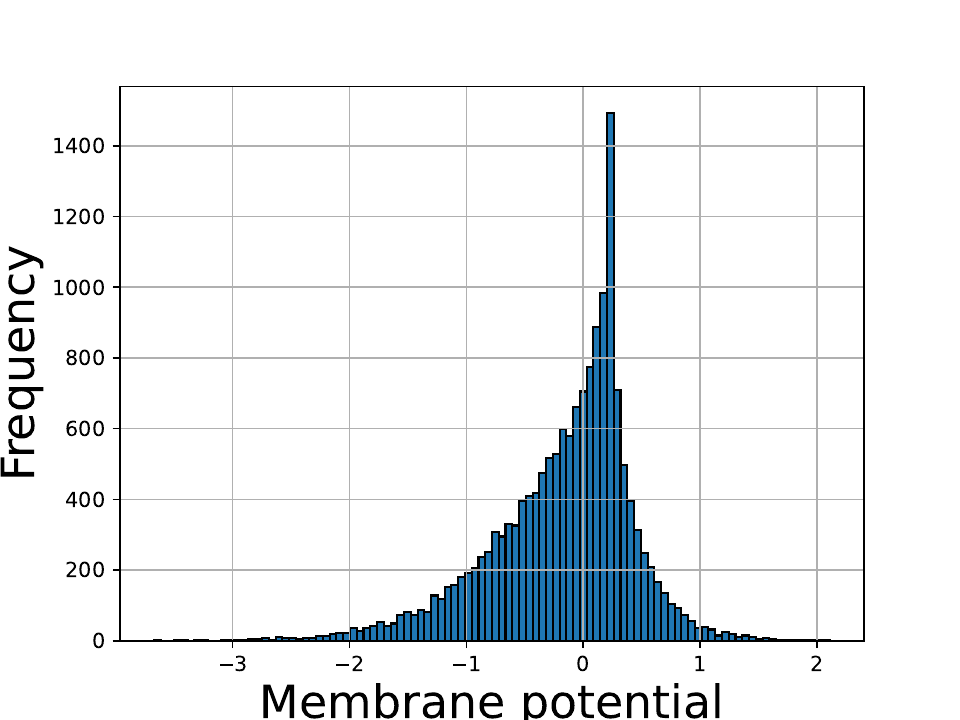}
        \caption{Layer $l=6$, $t=1$}
        \label{fig:stable2}
    \end{subfigure}
    \hfill
    \begin{subfigure}{0.48\textwidth}
        \centering
        \includegraphics[width=\textwidth]{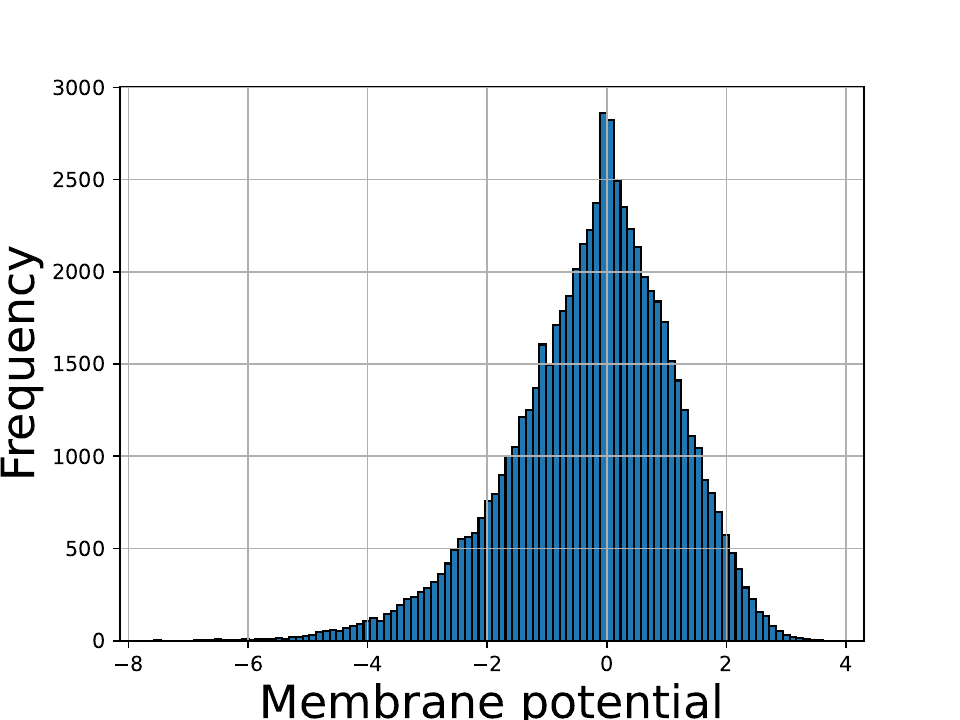}
        \caption{Layer $l=4$, $t=4$}
        \label{fig:stable1}
    \end{subfigure}
    \caption{Membrane potential distributions before firing for VGG16 model (pretrained on CIFAR100 dataset). In both cases, the SNN converted model has been calibrated for all the previous layers and for all the previous time steps (in (b)). A random IF neuron in the layer has been chosen for the presentation of the distributions.}
    \label{fig mem dist}
\end{figure}

\paragraph{Expected vs. True conversion error.} We have emphasized so far that the  expected outputs of ANN and the corresponding SNN can be made to be equal with temporal bias. However, the true error - the difference between outputs - when the same sample is passed through both the networks, may be very high. In order to further reduce the variance of the true error, we propose the implementation of our method channel-wise, rather than layer-wise, which has shown to be optimal in practice.

\section{Experiments Details}


\subsection{Datasets}

\noindent \textbf{CIFAR-10}: The CIFAR-10 dataset~\cite{krizhevsky2010cifar} comprises 60,000 color images of 32x32 pixels each, distributed across 10 distinct classes such as airplanes, cars, and birds, with each class containing 6,000 images. It is divided into a training set of 50,000 images and a test set of 10,000 images.

\noindent \textbf{CIFAR-100}: The CIFAR-100 dataset~\cite{krizhevsky2010cifar} features 60,000 color images of 32x32 pixels, but it spans 100 classes, maintaining the same distribution of 600 images per class. Like CIFAR-10, it includes 50,000 training samples and 10,000 testing samples.

\noindent \textbf{ImageNet}: The ImageNet dataset~\cite{deng2009imagenet} significantly expands the scope with 1,281,167 images across 1,000 classes for the training set, alongside a validation set and a testing set comprising 50,000 and 100,000 images, respectively, all categorized into the same 1,000 object classes. Unlike CIFAR datasets, ImageNet images vary in size and resolution, leading many applications to standardize the images to 256x256 pixels. The validation set is commonly repurposed as the test set in various applications.

\subsection{Post-Calibration Configuration and Setups}

\begin{table}[htbp]
    \caption{Hyperparameter Configuration for Post-Calibration Experiments.}
    \label{tab:config}
    \renewcommand\arraystretch{1.2}
	\centering
        \scalebox{0.90}{
        \resizebox{1.\textwidth}{!}{
    	\begin{tabular}{ccccccr}
            \toprule
    	  Dataset & Architecture & Method & \# Iters & Batch Size & \# Batches  & $\alpha$  \\ 
            \toprule
            \multirow{6}{*}{CIFAR-10} 
                & \multirow{3}{*}{ResNet-20}
                                        & Ours & 10 & 128 & 2 & 0.2 \\
                                        \cline{3-7}
                &                       & Ours (+QCFS) & 4 & 32 & 1 & 0.5 \\ 
                                        \cline{3-7}
                &                       & Our (+RTS) & 10 & 32 & 2 & 0.005 \\
                \cline{2-7}
                & \multirow{3}{*}{VGG-16} 
                                        & Ours & 10 & 128 & 2 & 0.2 \\ 
                                        \cline{3-7}
                &                       & Ours (+QCFS) & 10 & 10 & 1 & 0.5 \\ 
                                        \cline{3-7}
                &                       & Our (+RTS) & 3 & 32 & 3 & 0.005 \\
            \hline
            \multirow{6}{*}{CIFAR-100}            
                & \multirow{3}{*}{ResNet-20}
                                        & Ours & 1 & 128 & 2 & 0.2 \\
                                        \cline{3-7}
                &                       & Ours (+QCFS) & 4 & 32 & 1 & 0.5 \\ 
                                        \cline{3-7}
                &                       & Our (+RTS) & 3 & 32 & 3 & 0.005 \\
                \cline{2-7}
                & \multirow{3}{*}{VGG-16} 
                                        & Ours & 10 & 128 & 2 & 0.2 \\ 
                                        \cline{3-7}
                &                       & Ours (+QCFS) & 4 & 32 & 1 & 0.5 \\ 
                                        \cline{3-7}
                &                       & Our (+RTS) & 3 & 32 & 3 & 0.005 \\
            \hline
            \multirow{5}{*}{ImageNet} 
                & \multirow{3}{*}{ResNet-20}
                                        & Ours & 10 & 128 & 10 & 0.5 \\
                                        \cline{3-7}
                &                       & Ours (+QCFS) & 50 & 2 & 1 & 0.5 \\ 
                \cline{2-7}
                & \multirow{3}{*}{VGG-16} 
                                        & Ours & 10 & 128 & 10 & 0.5 \\ 
                                        \cline{3-7}
                &                       & Ours (+QCFS) & 50 & 2 & 1 & 0.5 \\ 
                                        \cline{3-7}
                &                       & Our (+RTS) & 3 & 128 & 10 & 0.5 \\                
            \bottomrule
    	\end{tabular}
        }
        }
\end{table}

Table~\ref{tab:config} presents the hyperparameter configurations employed in our calibration experiments, illustrating the settings tested across different datasets, architectures, and methods to optimize our post-calibration process.

\subsubsection{Vanilla \method}

In our experimental evaluations conducted on the CIFAR-10 and CIFAR-100 datasets, we adhered to the experimental setup previously described in the SNN Calibration~\cite{li2021free}, utilizing identical model architectures, namely ResNet-20, VGG-16, and MobileNet. Our data augmentation strategies encompassed random cropping with padding, random horizontal flipping, and, optionally, Cutout and AutoAugment techniques. This was followed by dataset-specific normalization procedures. 

For models that incorporate Batch Normalization (BN) layers, we employed Stochastic Gradient Descent (SGD) as the optimizer, with a momentum set to 0.9. The initial learning rate was configured to 0.1, with a cosine decay adjustment over 300 epochs. The weight decay parameter was finely tuned to \(5 \times 10^{-4}\) for the majority of cases; however, for models leveraging the MobileNet architecture, we specifically adjusted the weight decay to \(1 \times 10^{-4}\). In contrast, models devoid of BN layers were optimized with a learning rate of 0.01 and a weight decay of \(1 \times 10^{-4}\), to better accommodate the absence of normalization layers.

For our experiments on the ImageNet dataset, we utilized the same model architectures (ResNet-20, VGG-16, and MobileNet) and checkpoints as those detailed in the SNN Calibration~\cite{li2021free}. Our post-calibration \method involved utilizing batches of 10 mini-batch training data samples for step-by-step bias calibration on the ResNet-34 and VGG-16 models, with a batch size of 128 and \(\alpha=0.5\), and \(\alpha=0.4\) for MobileNet, specifically for the ImageNet dataset. For the CIFAR-10 and CIFAR-100 datasets, the process was conducted with 2 mini-batch training data samples, maintaining a batch size of 128 and setting \(\alpha=0.2\) for bias calibration on the ResNet-20 and VGG-16 models.

\subsubsection{\method integrated with QCFS}

We leveraged the pre-trained models obtained from QCFS~\cite{DBLP:conf/iclr/BuFDDY022} for our experiments on the ImageNet dataset. Consistent with the methodologies outlined in~\cite{DBLP:conf/iclr/BuFDDY022}, we trained ResNet-20 and VGG-16 models on the CIFAR-10 and CIFAR-100 datasets. In addition, a hyperparameter search was conducted specifically for the QCFS calibration process, aimed at identifying the optimal batch size. This search was performed while fixing parameter \(\alpha\) at 0.5.

\subsubsection{\method integrated with RTS}

In our experiments, we employed the checkpoints from the RTS as delineated in~\cite{DBLP:conf/iclr/DengG21} for both ResNet-34 and VGG-16 architectures on the ImageNet dataset, as well as ResNet-20 and VGG-16 on the CIFAR-10 and CIFAR-100 datasets. For the calibration process on the CIFAR-10 and CIFAR-100 datasets, we process 2 to 3 mini-batch training data samples with a batch size of 32 and a calibration parameter, \(\alpha\), set to $0.005$. Conversely, for experiments conducted on the ImageNet dataset, the calibration uses 10 mini-batch training data samples with an increased batch size of 128 and an \(\alpha\) value of $0.5$.

\section{Additional Experiments of Comparison between the Proposed Method \method and Previous Works on CIFAR-10 dataset}

In our evaluation on the CIFAR-10 dataset, as detailed in Table~\ref{tab:ann-snn-cifar10}, our method, particularly when integrated with QCFS on VGG-16, outperforms all baseline methods. For ResNet-20, our approach surpasses SlipReLU at $T\geq4$.

\begin{table}[htbp]
    \caption{Comparison between the proposed method \method and previous works on CIFAR-10 dataset}
    \label{tab:ann-snn-cifar10}
    \renewcommand\arraystretch{1.0}
    \centering
    \scalebox{0.9}
    {
    \begin{threeparttable}
    \begin{tabular}{cccccccccc}
    \toprule
    Architecture & Method & ANN & T=1 & T=2 & T=4 & T=8 & T=16 & T=32 & T=64\\
    \toprule
    \multirow{5}{*}{ResNet-20}
    & TSC & 91.47 & -- & -- & -- & -- & -- & -- & 69.38 \\
    & RTS & 93.25 & -- & -- & -- & -- & 92.41 & 93.30 & 93.55 \\
    & OPI & 92.74 & -- & -- & -- & 66.24 & 87.22 & 91.88 & 92.57 \\
    & QCFS & 91.77 & 62.43 & 73.20 & 83.75 & 89.55 & 91.62 & 92.24 & 92.35 \\
    & SlipReLU & 94.61 & \gc{80.99} & \gc{82.25} & 83.52 & 84.46 & 84.70 & 84.85 & 84.89 \\
    \cmidrule{2-10}
    & \textbf{Ours} & 96.95 & 59.75 & 74.87 & \gc{88.57} & \gc{93.73} & \gc{95.90} & \gc{96.74} & \gc{96.91} \\
    \cmidrule{2-10}
    & \textbf{Ours (+QCFS)} & 92.07 & 70.16 & 79.20 & 87.26 & 91.12 & 92.08 & 92.30 & 92.28 \\ 
    \hline
    \multirow{8}{*}{VGG-16}    
    & RMP & 93.63 & -- & -- & -- & -- & -- & 60.30 & 90.35 \\
    & RTS & 95.72 & -- & -- & -- & -- & -- & 76.24 & 90.64 \\
    & SNNC-AP & 95.72 & -- & -- & -- & -- & -- & 93.71 & 95.14 \\
    & RNL & 92.82 & -- & -- & -- & -- & 57.90 & 85.40 & 91.15 \\
    & RTS & 92.18 & -- & -- & -- & -- & 92.29 & 92.29 & 92.22 \\
    & OPI & 94.57 & -- & -- & -- & 90.96 & 93.38 & 94.20 & 94.45  \\
    & QCFS & 95.52 & 88.41 & 91.18 & 93.96 & 94.95 & 95.40 & 95.54 & 95.55 \\
    & SlipReLU & 93.02 & 88.17 & 89.57 & 91.08 & 92.26 & 92.96 & 93.19 & 93.25 \\
    \cmidrule{2-10}    
    & \textbf{Ours} & 95.69 & 61.76 & 83.10 & 90.47 & 93.33 & 94.83 & 95.40 & 95.60 \\
    \cmidrule{2-10}
    & \textbf{Ours (+QCFS)} & 95.92 & \gc{90.36} & \gc{92.08} & \gc{94.67} & \gc{95.57} & \gc{95.89} & \gc{95.98} & \gc{96.06} \\
    \bottomrule
    \end{tabular}
    \end{threeparttable}
    }
\end{table}

\section{Additional Experiments of Comparison with the Self-Referential Baseline}

Our study delves into both the vanilla effects and the synergistic effects of integrating \method with established methods such as QCFS and RTS, spotlighting \method's adaptability and strength not only as a standalone strategy but also when synergized with other optimization methods. This fusion particularly underscores remarkable enhancements in accuracy particularly at earlier timesteps, thus demonstrating \method's capability to expedite model learning and efficiency.

\subsubsection{Vanilla \method}

Table~\ref{tab:cifar-10-100-snnc} provides an overview of the comparative analysis between our proposed method (\method) and the SNNC-AP approach across various network architectures (ResNet-20, VGG-16, MobileNet) and for both CIFAR-10 and CIFAR-100 datasets. We conduct model calibrations exclusively on architectures featuring Batch Normalization (BN) layers. The proposed method (\method) shows a robust performance, either matching or surpassing the SNNC-AP approach across different timestep configurations ($T=1$ to $T=128$). 

\begin{table}[htbp]
\caption{Comparison of our algorithm with SNNC-AP on CIFAR-10 and CIFAR-100. } 
\label{tab:cifar-10-100-snnc}
\centering
\renewcommand\arraystretch{1.2}
\scalebox{0.9}{
\begin{tabular}{lcccccccccc}
\toprule
Method & ANN Acc. & $T=1$ & $T=2$ & $T=4$ & $T=8$ & $T=16$ & $T=32$ & $T=64$ & $T=128$ \\
\toprule
\multicolumn{10}{c}{\textbf{ResNet-20, CIFAR100}} \\
\midrule
SNNC-AP & 81.89 & 18.55 & 32.50 & 54.08 & \gc{72.69} & 78.73 & 81.03 & \gc{81.88} & 81.94 \\
Ours & 81.89 & \gc{19.96} & \gc{38.19} & \gc{58.08} & 71.74 & \gc{78.80} & \gc{81.09} & 81.79 & 81.94 \\
\midrule
\multicolumn{10}{c}{\textbf{VGG-16, CIFAR100}} \\
\midrule
SNNC-AP & 77.87 & 29.36 & 34.02 & 50.39 & 65.00 & 71.70 & 75.54 & 77.10 & \gc{77.86} \\
Ours & 77.87 & \gc{32.79} & \gc{48.99} & \gc{60.68} & \gc{69.52} & \gc{74.05} & \gc{76.39} & \gc{77.34} & 77.73 \\
\midrule
\multicolumn{10}{c}{\textbf{MobileNet, CIFAR100}} \\
\midrule
SNNC-AP & 73.21 & 1.12 & 2.72 & 3.29 & 4.24 & 16.61 & 56.29 & 68.49 & 71.66 \\
Ours & 73.21 & \gc{1.47} & \gc{2.77} & \gc{5.40} & \gc{14.55} & \gc{44.92} & \gc{63.22} & \gc{69.94} & \gc{72.25} \\
\midrule
\multicolumn{10}{c}{\textbf{ResNet-20, CIFAR10}} \\
\midrule
SNNC-AP & 96.95 & 56.48 & 73.70 & 85.56 & 93.34 & 95.64 & 96.64 & 96.72 & 96.94 \\
Ours & 96.95 & \gc{59.75} & \gc{74.87} & \gc{88.57} & \gc{93.73} & \gc{95.90} & \gc{96.74} & \gc{96.91} & \gc{96.96} \\
\midrule
\multicolumn{10}{c}{\textbf{VGG-16, CIFAR10}} \\
\midrule
SNNC-AP & 95.69 & 60.30 & 75.37 & 84.34 & 92.55 & 93.81 & 95.06 & 95.53 & 95.54 \\
Ours & 95.69 & \gc{61.76} & \gc{83.10} & \gc{90.47} & \gc{93.33} & \gc{94.83} & \gc{95.40} & \gc{95.60} & \gc{95.64} \\
\midrule
\multicolumn{10}{c}{\textbf{MobileNet, CIFAR10}} \\
\midrule
SNNC-AP & 94.11 & 13.91 & 18.68 & 20.18 & 27.64 & 61.80 & 87.46 & \gc{92.67} & 93.41 \\
Ours &  94.11 & \gc{15.93} & \gc{26.64} & \gc{32.27} & \gc{46.97} & \gc{78.59} & \gc{89.64} & 92.58 & \gc{93.58} \\
\bottomrule
\end{tabular}
}
\end{table}

\subsubsection{\method integrated with QCFS}

Table~\ref{tab:cifar100-qcfs} further expands our comparative analysis, scrutinizing the impact of varying batch sizes, ranging from $8$ to $32$, for models such as ResNet-20 and VGG-16 on CIFAR-10 and CIFAR-100 datasets. This batch size optimization exercise underscores the critical role of hyperparameter fine-tuning in amplifying the model's early learning phase, thereby facilitating faster and more efficient convergence. Our method surpasses all baseline metrics for timesteps up to $T=64$, indicating the broad applicability and potential of hyperparameter tuning in enhancing the generalization capability of models.

\begin{table}[htbp]
\caption{Comparison between QCFS and ours integrated with QCFS on CIFAR-10 and CIFAR-100.}
\label{tab:cifar100-qcfs}
\centering
\renewcommand\arraystretch{1.2}
\scalebox{0.92}{
\begin{threeparttable}
\begin{tabular}{@{}clllllllllll@{}}
\toprule
Architecture & Method & \textbf{BS} & ANN & T=1 & T=2 & T=4 & T=8 & T=16 & T=32 & T=64 & T=128 \\
\toprule
\multicolumn{12}{c}{\textbf{CIFAR-10}} \\
\hline
\multirow{5}{*}{ResNet-20} & QCFS & -- & 92.07 & 62.91 & 73.00 & 83.72 & 89.57 & 91.67 & 92.22 & 92.23 & 92.18 \\
                            \cmidrule{2-12}
                            & \multirow{4}{*}{Ours (+QCFS)} & 8 & 92.07 & 68.79 & 78.30 & 86.90 & 90.94 & 92.01 & \gc{92.31} & 92.21 & \gc{92.26} \\
                            &         & 16 & 92.07 & 68.56 & 78.27 & 86.84 & \gc{91.22} & \gc{92.19} & 92.15 & 92.20 & 92.22 \\
                            &         & 32 & 92.07 & \gc{70.16} & \gc{79.20} & \gc{87.26} & 91.12 & 92.08 & 92.30 & \gc{92.28} & 92.23 \\
                            &         & 10 & 92.07 & 67.89 & 77.4 & 86.84 & 91.16 & 92.05 & 92.25 & 92.27 & 92.22 \\
\hline
\multirow{5}{*}{VGG-16} & QCFS & -- & 95.92 & 89.04 & 91.42 & 94.33 & 95.21 & 95.65 & 95.87 & 95.99 & 95.99 \\
                            \cmidrule{2-12}
                            & \multirow{4}{*}{Ours (+QCFS)} & 8 & 95.92 & \gc{90.55} & \gc{92.16} & 94.47 & 95.39 & 95.76 & 95.94 & 95.97 & 95.97 \\
                            &         & 16 & 95.92 & 90.16 & 91.64 & 94.06 & 95.21 & 95.71 & 95.92 & 95.99 & 96.02 \\
                            &         & 32 & 95.92 & 90.33 & 91.85 & 94.42 & 95.47 & 95.79 & 95.92 & 96.00 & 95.99 \\
                            &         & 10 & 95.92 & 90.36 & 92.08 & \gc{94.67} & \gc{95.57} & \gc{95.89} & \gc{95.98} & \gc{96.06} & \gc{96.04} \\
\hline
\multicolumn{12}{c}{\textbf{CIFAR-100}} \\
\hline
\multirow{5}{*}{ResNet-20} & QCFS & -- & 67.09 & 10.64 & 15.34 & 27.87 & 49.53 & 63.61 & 67.04 & 67.87 & \gc{67.86} \\
                            \cmidrule{2-12}
                            & \multirow{4}{*}{Ours (+QCFS)} & 8 & 67.09 & 14.98 & 23.30 & 40.61 & 58.42 & 65.37 & 67.29 & 67.61 & 67.60 \\
                            &         & 16 & 67.09 & \gc{16.09} & 24.60 & 41.88 & 58.53 & \gc{65.77} & 67.23 & 67.66 & 67.73 \\
                            &         & 32 & 67.09 & 15.88 & \gc{25.09} & \gc{42.10} & \gc{58.81} & 65.60 & 67.37 & \gc{67.89} & 67.80 \\
                            &         & 10 & 67.09 & 14.10 & 22.13 & 38.80 & 57.43 & 65.18 & \gc{67.48} & 67.83 & 67.72 \\
\hline
\multirow{5}{*}{VGG-16} & QCFS & -- & 76.21 & 49.09 & 63.22 & 69.29 & 73.89 & 75.98 & 76.53 & 76.54 & \gc{76.60} \\
                            \cmidrule{2-12}
                            & \multirow{4}{*}{Ours (+QCFS)} & 8 & 76.21 & 61.48 & 66.93 & 71.43 & 74.51 & 75.99 & 76.41 & 76.54 & 76.59 \\
                            &         & 16 & 76.21 & 61.78 & 66.54 & \gc{71.50} & 74.82 & 76.10 & 76.46 & \gc{76.62} & 76.54 \\
                            &         & 32 & 76.21 & \gc{62.22} & \gc{67.77} & 71.47 & \gc{75.12} & \gc{76.22} & \gc{76.48} & 76.48 & 76.48 \\
                            &         & 10 & 76.21 & 61.81 & 66.58 & 71.21 & 74.52 & 75.91 & 76.25 & 76.38 & 76.45 \\
\hline
\end{tabular}
\end{threeparttable}
}
\end{table}

\subsubsection{\method integrated with RTS}

Table~\ref{tab:imagenet-rts} and Table~\ref{tab:cifar-10-100-rts} show improved RTS performance on both ImageNet and CIFAR-100 for all evaluated \(T\) values after applying our post-calibration method \method. A significant increase in accuracy is noted for ResNet-20 starting at \(T=8\), where accuracy surges from 1.78\% to 47.00\%. However, the performance gains at lower \(T\) values, specifically \(T=4\) and \(T=8\) for CIFAR-10 and \(T=4\) for CIFAR-100, are constrained. This observation is consistent with the limitations noted in QCFS's performance at these \(T\) values.

\begin{table}[htbp]
\caption{Comparison between RTS and Ours (+RTS) on ImageNet.}
\label{tab:imagenet-rts}
\centering
\renewcommand\arraystretch{1.3}
\scalebox{0.85}{
\begin{threeparttable}
\begin{tabular}{@{}clllllllllll@{}}
\toprule
Architecture & Dataset & Method & ANN & T=1 & T=2 & T=4 & T=8 & T=16 & T=32 & T=64 & T=128\\ 
\toprule
\multirow{2}{*}{VGG-16}
&ImageNet & RTS & 72.16 & 0.10 & 2.38 & 4.78 & 26.49 & 56.83 & 67.96 & 70.93 & 71.86 \\
&ImageNet & \textbf{Ours (+RTS)} & 72.16 & \gc{0.61} & \gc{14.10} & \gc{29.61} & \gc{55.22} & \gc{67.14} & \gc{70.74} & \gc{71.86} & \gc{72.13} \\
\bottomrule
\end{tabular}
\end{threeparttable}
}
\end{table}

\begin{table}[htbp]
\caption{Comparison between RTS and Ours (+RTS) on CIFAR-100.}
\label{tab:cifar-10-100-rts}
\centering
\renewcommand\arraystretch{1.3}
\scalebox{0.95}{
\begin{threeparttable}
\begin{tabular}{@{}clllllllll@{}}
\toprule
Architecture & Dataset & Method & ANN & T=4 & T=8 & T=16 & T=32 & T=64 & T=128\\ 
\toprule
\multicolumn{10}{c}{\textbf{CIFAR-100}} \\
\hline
\multirow{2}{*}{VGG-16}
&CIFAR-100 & RTS & 70.38  & 2.00 & 28.52 & 63.90 & 70.26 & 70.37 & 70.51 \\
&CIFAR-100 & \textbf{Ours (+RTS)} & 70.38 & \gc{3.67} & \gc{30.06} & \gc{64.24} & \gc{70.28} & \gc{70.46} & \gc{70.54} \\
\hline
\multirow{2}{*}{ResNet-20}
&CIFAR-100 & RTS & 69.80  & 1.00 & 1.78 & 62.43 & 68.15 & 69.27 & 69.58 \\
&CIFAR-100 & \textbf{Ours (+RTS)} & 69.80  & 1.00 & \gc{47.00} & \gc{63.40} & \gc{68.76} & \gc{69.43} & \gc{69.68} \\ 
\hline
\multicolumn{10}{c}{\textbf{CIFAR-10}} \\
\hline
\multirow{2}{*}{VGG-16}
&CIFAR-10 & RTS & 92.18 & 33.11 & 88.91 & 92.26 & 92.25 & 92.18 & 92.23 \\
&CIFAR-10 & \textbf{Ours (+RTS)} & 92.18 & \gc{39.16} & \gc{88.95} & \gc{92.48} & \gc{92.27} & \gc{92.29} & \gc{92.26} \\
\hline
\multirow{2}{*}{ResNet-20}
&CIFAR-10 & RTS & 93.25 & 10.00 & 10.00 & 52.11 & 92.41 & 92.88 & 93.08 \\
&CIFAR-10 & \textbf{Ours (+RTS)} & 93.25 & 10.00 & 10.00 & \gc{91.16} & \gc{92.72} & \gc{93.10} & \gc{93.16} \\
\bottomrule
\end{tabular}
\end{threeparttable}
}
\end{table}

\end{document}